\documentclass{article}

\usepackage{fullpage}

\usepackage{amsmath,amsfonts,amsthm,amssymb}

\usepackage{graphics, graphicx, xcolor}
\usepackage{hyperref}
\usepackage{url}
\usepackage{latexsym}
\usepackage{array}
\usepackage{verbatim}
\usepackage[shortlabels,]{enumitem}
\usepackage{algorithm}
\usepackage{algorithmic}
\usepackage{multirow}

\usepackage{wrapfig}

\usepackage{rotating}

\usepackage{booktabs}
\usepackage{cellspace}
\usepackage{cleveref}
\usepackage{csquotes}

\newtheorem{theorem}{Theorem}

\newtheorem{definition}[theorem]{Definition}

\DeclareMathOperator*{\arginf}{arg\,inf}

\DeclareMathOperator*{\argsup}{arg\,sup}
\DeclareMathOperator{\sgn}{sgn}
\DeclareMathOperator{\Prtxt}{Pr}

\newcommand{\RR}{\mathbb{R}}      
\newcommand{\evp}[2]{\mathbb{E}_{#2} \left[#1\right]} 
\newcommand{\abs}[1]{\left| #1 \right|}

\newcommand{\prp}[2]{\Prtxt_{#2} \left(#1\right)}

\newcommand{\cH}{\mathcal{H}}
\newcommand{\cX}{\mathcal{X}}

\newcommand{\lrp}[1]{\left(#1\right)}
\newcommand{\lrb}[1]{\left[#1\right]}

\title{
Linking Generative Adversarial Learning and Binary Classification
}

\author{
Akshay Balsubramani\\
\url{abalsubr@stanford.edu}
}

\begin{document}

\maketitle

\begin{abstract}
In this note, we point out a basic link between generative adversarial (GA) training and binary classification -- any powerful discriminator essentially computes an ($f$-)divergence between real and generated samples. 
The result, repeatedly re-derived in decision theory, has implications for GA Networks (GANs), providing an alternative perspective on training $f$-GANs by designing the discriminator loss function.
\end{abstract}


\section{Generating a Distribution}
Imagine we are given real data with distribution $P_r (x)$ over a feature space $\cX$, 
and wish to learn a distribution $P_g (x)$ that is as ``close" as possible to $P_r$. 
Closeness will be measured by some divergence function $D ( \cdot, \cdot )$ between probability distributions. 
So the generator is typically solving 
\begin{align}
\label{eq:genprob}
P_g^* &= \arginf_{P_g} D (P_g , P_r)
\end{align}
where $P_g$ is in some class of distributions specified by the generator. 
This manuscript considers $D(\cdot,\cdot)$ to be an $f$-divergence:
\begin{definition}\textbf{$f$-divergence.} For any convex $f (t)$ with $f(1) = 0$, 
define the $f$-divergence of $P_g$ to $P_r$
\footnote{Note that $D_f (P_g , P_r)$ need not be positive.} 
as 
$$ D_f (P_g , P_r) = \evp{ f \lrp{\frac{P_g (x)}{P_r (x)}}}{x \sim P_r} $$
\end{definition}


Here we examine the method of highlighting differences between $P_r$ and $P_g$ 
by feeding them to a binary classifier (discriminator) 
with corresponding labels $y = \pm 1$, 
in equal proportion as for a typical GAN setup, 
so that the data input to the discriminator are assigned a positive label if they are real data: 
\begin{align*}
p := \prp{y = +1}{} = \frac{1}{2}
\quad,\quad
P_r (x) = \prp{x \mid y = + 1}{}
\quad,\quad
P_g (x) = \prp{x \mid y = - 1}{}
\end{align*}
Recall that any two-class loss function can equivalently be written in terms of \emph{partial losses} $\ell_{+} (g)$ and $\ell_{-} (g)$; 
these are the losses with respect to true labels $\pm 1$ respectively, as a function of the label prediction $g$. 

The discrimination problem is to find a function $h$ in some model class $\cH$ that attempts to minimize some loss on average over the data: 
\begin{align}
\label{eq:lossmin}
\inf_{h \in \cH} \evp{ \ell (y , h(x)) }{(x, y)}
\end{align}
The \emph{generative} view of binary classification \cite{RW11}
writes this in terms of the class-conditional distributions $\prp{x \mid y = \pm 1}{}$: 
\begin{align}
\label{eq:lossclasscondgan}
\evp{ \ell (y , h(x)) }{y} &= \prp{x , y = +1}{} \ell_{+} (h(x)) + \prp{x , y = -1}{} \ell_{-} (h(x)) \nonumber \\
&= p \prp{x \mid y = +1}{} \ell_{+} (h(x)) + (1 - p) \prp{x \mid y = -1}{} \ell_{-} (h(x)) \nonumber \\
&= \frac{1}{2} \lrb{ P_r (x) \ell_{+} (h(x)) + P_g (x) \ell_{-} (h(x)) }
\end{align}

The optimization problem \eqref{eq:lossmin} is standard in binary classification. 
Typically, $\cH$ is chosen to be a fairly rich class of deep binary classifiers. 
This means that its performance is close to the Bayes risk, i.e. 
the minimum risk over measurable functions 
$\inf_{h} \evp{ \ell (y , h(x)) }{(x, y)}$ \cite{RW11}. 
So the excess risk 
\begin{align*}
\epsilon (\cH) := \inf_{h \in \cH} \evp{ \ell (y , h(x)) }{(x, y)} - \inf_{h} \evp{ \ell (y , h(x)) }{(x, y)}
\end{align*}
is small.


\section{Main Result}

\begin{theorem}
\label{thm:binlosses}
Take any loss function $\ell_{\pm}$ and any model class $\cH$. 
Define $\displaystyle f (s) := \sup_{\alpha} \lrp{ - \ell_{+} (\alpha) - s \ell_{-} (\alpha) } $. 
This is a maximum of linear functions, so it is convex. Then 
\begin{align*}
\inf_{h \in \cH} \evp{ \ell (y , h(x)) }{(x, y)} = 
- \frac{1}{2} D_{f} (P_g , P_r) + \epsilon (\cH)
\end{align*}
\end{theorem}

Changing the model class $\cH$ only changes the second term of Thm. \ref{thm:binlosses}.
Therefore, when $\cH$ is rich enough that the excess risk $\epsilon (\cH)$ is small, 
the loss function $\ell$ of the discrimination problem corresponds almost exactly to an $f$-divergence.

\subsection{GA Training Solves the Generation Problem with $f$-Divergences}

Revisiting \eqref{eq:genprob}, to find $P_g^*$ to be ``close'' to $P_r$ 
under some $f$-divergence $D_{f}$, one could solve
\begin{align*}
P_g^* &= \arginf_{P_g} D_{f} (P_g , P_r) 
= \argsup_{P_g} \lrb{ - D_{f} (P_g , P_r) }
= \argsup_{\prp{x \mid y=-1}{}} \lrb{ \inf_{h \in \cH} \evp{ \ell (y , h(x)) }{(x, y)} - \epsilon(\cH) } \\
&\approx \argsup_{\prp{x \mid y=-1}{}} \lrb{ \inf_{h \in \cH} \evp{ \ell (y , h(x)) }{(x, y)} }
\end{align*}
So the adversarial game interaction between the generator and discriminator emerges as the solution to the generation problem for powerful enough discriminators, for any $\ell, \cH$. 


\subsection{Examples}

Table 1 
shows the correspondence between $\ell$ and $f$ for several common $f$-divergences. 
Similar lists can be found in \cite{NWJ09, RW11}.

In the GA setup, the variable $s$ is always a function over the data space $\cX$. 
The maximizing $\alpha$ in $\argsup_{\alpha} \lrp{ - \ell_{+} (\alpha) - s \ell_{-} (\alpha) }$ is a function of $s$; 
as a function of the data $\alpha (x)$, it is the optimal discriminator 
$h^* (x) = h^* (s(x))$.

\begin{table}[!htbp]
\centering
\small
\begin{tabular}{| Sc || Sc | p{2cm} | Sc || p{3cm} |} \hline \hline 
Loss $\ell$ & Partial losses & $f (s)$  &  $h^* (s)$ & $f$-divergence \\ \hline \hline
 0-1 & $\ell_{\pm} (g) = \frac{1}{2}\lrp{1 \mp g}$  & $\frac{1}{2} \abs{s-1}$ &   $\sgn (s - 1)$ 
 & Total variation dist. \\ \hline
   Log & $\ell_{\pm} (g) = \ln \lrp{\frac{2}{1 \pm g}}$  & $- \ln \lrp{1 + s} - s \ln \lrp{\frac{1 + s}{s}}$  & $\frac{1-s}{1+s}$ 
  & Jensen-Shannon dist. \\ \hline
   Square & $\ell_{\pm} (g) = \lrp{1 \mp g}^2$  & $- \frac{s}{1+s} + \frac{1}{2}$ & $\frac{1-s}{1+s}$ 
   & Triangular discrimination dist. \\ \hline
   CW (param. $c$) & $\ell_{\pm} (g) = \lrp{\frac{1}{2} - \lrp{\frac{1}{2} - c} } \lrp{1 \mp g}$  & $\abs{1-c - cs} - cs + c - \abs{1-2c}$ & $\sgn (1 - c - cs)$
   & -- \\ \hline
   Exponential & $\ell_{\pm} (g) = \exp( \mp g)$  & $- 2 \sqrt{s} + 2$   & $- \frac{1}{2} \ln s$ 
   & Hellinger dist. \\ \hline
    ``Boosting" & $\ell_{\pm} (g) = \sqrt{\frac{1 \mp g}{1 \pm g}}$  & $- 2 \sqrt{s} + 2$  & $\frac{1-s}{1+s}$
   & Hellinger dist. \\ \hline
\end{tabular}
\label{tab:alllosses}
\caption{Some discriminator losses, with corresponding $f$-divergences.}
\end{table}

\section{Related Work}

The most related work to this manuscript is the $f$-GAN approach of \cite{NCT16fgan}, to our knowledge. 
This solves the same problem of minimizing the $f$-divergence to the true distribution, 
but by changing the discriminator objective from the binary classification risk, 
(in contrast to Thm. \ref{thm:binlosses} which just interprets the risk). 
The key fact is that a convex function $f$ has a well-defined \emph{convex conjugate} function $f^*$ 
such that $f(u) = \sup_{t \in \RR} \lrb{tu - f^* (t)}$, 
so that the following is true\footnote{Ignoring conjugacy domain issues for simplicity.}:
\begin{align}
\label{eq:fganvar}
D_{f} (P_r , P_g) 
&= \evp{ \sup_{t} \lrp{t \frac{P_r (x)}{P_g (x)} - f^* (t)} }{x \sim P_g} 
= \sup_{t} \lrp{ \evp{t \frac{P_r (x)}{P_g (x)} - f^* (t)}{x \sim P_g} } \nonumber \\
&= \sup_{h} \lrb{ \evp{h(x)}{x \sim P_r} - \evp{f^* (h(x))}{x \sim P_g} } 
\geq \sup_{h \in \cH} \lrb{ \evp{h(x)}{x \sim P_r} - \evp{f^* (h(x))}{x \sim P_g} }
\end{align}
\cite{NCT16fgan} use this bound from \cite{NWJ10}. 
It is quite tight when $\cH$ is rich, exactly when Thm. \ref{thm:binlosses} is strong, 
though the order of the arguments is switched.
\footnote{Note that \eqref{eq:fganvar} puts the modeled (generated) distribution in the second argument 
rather than the first; the two orderings are related by Csisz\'{a}r duality \cite{RW11}. 
Our proof of Theorem \ref{thm:binlosses} can be followed to prove an exact analogue of Theorem \ref{thm:binlosses} viewing the discriminator risk 
as $D_{f} (P_r , P_g)$ with the arguments interchanged as in \eqref{eq:fganvar}. 
The analogue result only differs in the definition of the convex function generating the divergence, which is 
$\sup_{\alpha} \lrp{ - \ell_{-} (\alpha) - s \ell_{+} (\alpha) } $ instead of $f$ as defined in Thm. \ref{thm:binlosses}. 
In this manuscript, we follow the convention of taking the real distribution to be the second argument with respect to which we measure the ``excess description length" of using $P_g$.}

More broadly, since the original GAN paper \cite{GAN14},  
the GA approach has enjoyed a string of recent empirical successes 
with very rich model classes $\cH$ \cite{DCGAN15, BEGAN17}, in accordance with Thm. \ref{thm:binlosses}. 


\section{Summary}

The correspondences here fundamentally link generative adversarial training and the generation problem, 
and most are well known in decision theory. 
However, within the GAN literature they do not appear well known and lack references, 
which we address in this note.

\bibliography{GANwriteup}
\bibliographystyle{alpha}

\section{Proofs}

This section relates binary classification loss functions to $f$-divergences, recapitulating \cite{LV06, NWJ09}.


\begin{proof}[Proof of Theorem \ref{thm:binlosses}]
From \eqref{eq:lossclasscondgan}, 
if $\mu$ is the base measure over $\cX$,
\begin{align*}
\inf_{h} \;&\evp{ \ell (y , h(x)) }{(x, y)} 
= \inf_{h} \evp{ \evp{ \ell (y , h(x))}{y} }{x} 
= \frac{1}{2} \inf_{h} \evp{ P_r (x) \ell_{+} (h(x)) + P_g (x) \ell_{-} (h(x)) }{x \sim \mu} \\
&= \frac{1}{2} \inf_{h} \evp{ P_r (x) \lrp{ \ell_{+} (h(x)) + \frac{P_g (x)}{P_r (x)} \ell_{-} (h(x)) } }{x \sim \mu} 
= \frac{1}{2} \inf_{h} \evp{ \ell_{+} (h(x)) + \frac{P_g (x)}{P_r (x)} \ell_{-} (h(x)) }{x \sim P_r} \\
&= \frac{1}{2} \evp{ \inf_{\alpha} \lrp{ \ell_{+} (\alpha) + \frac{P_g (x)}{P_r (x)} \ell_{-} (\alpha) } }{x \sim P_r}
= - \frac{1}{2} \evp{ \sup_{\alpha} \lrp{ - \ell_{+} (\alpha) - \frac{P_g (x)}{P_r (x)} \ell_{-} (\alpha) } }{x \sim P_r} \\
&= - \frac{1}{2} D_{f} (P_g , P_r)
\end{align*}
Adding $\epsilon(\cH)$ to both sides proves the result.
\end{proof}

\end{document}